\documentclass[11pt,letterpaper]{article} 

\usepackage[margin=0.9in,includefoot=false,
            footskip=0.4in,nohead=true]{geometry}
\usepackage{url}
\usepackage{hyperref}
\usepackage{color}
\usepackage{subfigure}
\usepackage[numbers,sort&compress]{natbib}
\usepackage{verbatim}
\usepackage{authblk}
\usepackage[frak=mma,scr=rsfs,cal=mma]{mathalfa}
\DeclareMathAlphabet{\mathcal}{OMS}{cmsy}{b}{n}
\usepackage[pdftex]{graphicx}
\usepackage{tabularx}
\usepackage[margin=6pt,font={sf,footnotesize},labelfont=bf]{caption}
\usepackage{times}
\usepackage{amsmath}
\usepackage{amssymb}
\usepackage{dsfont}
\definecolor{linkcolour}{rgb}{0,0.2,0.6}
\hypersetup{colorlinks,breaklinks,linkcolor=linkcolour,citecolor=linkcolour,filecolor=linkcolour,urlcolor=linkcolour}
\usepackage{sectsty}
\usepackage{bm}
\sectionfont{\large \sffamily\bfseries}
\subsectionfont{\normalsize \sffamily\bfseries}
\subsubsectionfont{\it \sffamily}


\usepackage{amssymb, graphicx, dsfont, setspace, amsfonts}
\usepackage{algorithm}
\usepackage{algorithmic}
\usepackage{amsthm}

\renewcommand{\subset}{\subseteq}
\DeclareMathOperator{\cost}{cost}

\newcommand{\st}{\colon\,}
\newcommand{\cee}{\mathcal{C}}

\newcommand{\caze}[2]{\textbf{Case {#1}:} \textit{#2}}
\newcommand{\sizeof}[1]{\left\lvert{#1}\right\rvert}
\newtheorem{observation}{Observation}
\newtheorem{lemma}{Lemma}
\newtheorem{definition}{Definition}

\newcommand{\wpp}{w^+}
\newcommand{\wmm}{w^-}

\begin{document}

\title{\Large\sffamily\bfseries A new correlation clustering method for cancer mutation analysis}

{\normalsize
\author[1,2,$\dagger$]{\sffamily Jack P. Hou}
\author[3,$\dagger$]{\sffamily Amin Emad}
\author[3]{\sffamily Gregory J. Puleo}
\author[1,4,*]{\sffamily Jian Ma}
\author[3,*]{\sffamily Olgica Milenkovic}
\affil[1]{\small \sffamily Department of Bioengineering and Carl R. Woese Institute for Genomic Biology, University of Illinois at Urbana-Champaign}
\affil[2]{\small \sffamily Medical Scholars Program, University of Illinois at Urbana-Champaign}
\affil[3]{\small \sffamily Department of Electrical and Computer Engineering and Coordinated Science Lab, University of Illinois at Urbana-Champaign}
\affil[4]{\small \sffamily Present Affiliation: School of Computer Science, Carnegie Mellon University}
\affil[$\dagger$]{\small \sffamily These authors contributed equally.}
\affil[*]{\small \sffamily Corresponding authors:
\href{mailto:jianma@cs.cmu.edu}{jianma@cs.cmu.edu} (J.M.) and \href{mailto:milenkov@illinois.edu}{milenkov@illinois.edu} (O.M.)}

\date{\vspace{-5ex}}

\maketitle

\begin{abstract} 
\noindent Cancer genomes exhibit a large number of different alterations that affect many genes in a diverse manner. 
It is widely believed that these alterations follow combinatorial patterns that have a strong connection with the underlying molecular interaction networks and functional pathways. A better understanding of the generative mechanisms behind the mutation rules and their influence on gene communities is of great importance for the process of driver mutations discovery and for identification of network modules related to cancer development and progression.
We developed a new method for cancer mutation pattern analysis based on a constrained form of correlation clustering. Correlation clustering is an agnostic learning method that can be used for general community detection problems in which the number of communities or their structure is not known beforehand.
The resulting algorithm, named C$^3$, leverages mutual exclusivity of mutations, patient coverage, and driver network concentration principles; it accepts as its input a user determined combination of heterogeneous patient data, such as that available from TCGA (including mutation, copy number, and gene expression information), and creates a large number of clusters containing mutually exclusive mutated genes in a particular type of cancer.
The cluster sizes may be required to obey some useful soft size constraints, without impacting the computational complexity of the algorithm. To test C$^3$, we performed a detailed analysis on TCGA breast cancer and glioblastoma data and showed that our algorithm outperforms the state-of-the-art CoMEt method in terms of discovering mutually exclusive gene modules and identifying driver genes.
Our C$^3$ method represents a unique tool for efficient and reliable identification of  mutation patterns and driver pathways in large-scale cancer genomics studies.
\end{abstract}

\section{Introduction}

Rapid advances in high-throughput sequencing technologies have provided unique opportunities for analyzing large numbers of cancer genomes. 
However, the complexity of genomic alterations in cancer remains a challenge that has to be overcome in order to fully characterize the functional roles of various mutations. 
Cancer genomes often exhibit a large number of different mutations that affect genes in diverse manners. 
But the vast majority of these mutations do not have significant impact on tumorigenesis~\citep{hanahan2011}. 
A central question in cancer genomics is how to distinguish ``driver'' mutations, which contribute to tumorigenesis, from functionally neutral ``passenger'' mutations. 
Such driver mutations (e.g., point mutations or copy number changes) are of critical importance in elucidating key biological pathways perturbed in cells that eventually lead to cancer.

Many computational methods have been developed to facilitate the discovery of driver mutations~\citep{carter2009,dees2012,lawrence2013,gonzalez2012,manolakos2014camodi}. 
Due to the high level of inter-tumor heterogeneity, two patients with the same cancer may have vastly different drivers and most cancer mutations occur in very low frequency in the patient population.
Therefore, approaches relying on simple recurrence or frequency of mutations do not work well in practice. 
Recently, pathway-based and network-based models have been shown to be effective not only in determining common driver mutations and mutation patterns, but also in pinpointing the key biological pathways and subnetworks affected by driver mutations ~\citep{pe2011,ng2012,bashashati2012,paull2013,hou2014}.
Such methods have a unique advantage in so far that in addition to mutation analysis, they take into account gene interactions as an added source of prior knowledge.

In parallel, a number of methods have been proposed to identify driver pathways, i.e., groups of genes that may interact together in combinatorial patterns to promote tumorigenesis. 
The authors of~(\citealp{Ciriello12}) described a method called MEMo, and subsequently used it to show that mutually exclusive modules based on known networks can aid in determining groups of genes that contribute to tumorigenesis.
These gene groups, or modules, are jointly highly recurrent, have similar pathway impact in terms of biological processes, and their corresponding mutations tend to be mutually exclusive, meaning that very often only one gene in each gene group is mutated at a given time in any given patient. This mutual exclusivity rule in cancer pathways is supported by the observations that, in general, one mutated gene suffices to perturb the function of its corresponding pathway. Multiple mutations would require significantly higher energy investments on the part of cancer cells, and are hence selected against. 
Dendrix~\citep{vandin2012} (and later Multi-Dendrix~\citep{Leiserson2013}) was developed to identify driver pathways {\it de novo} using mutual exclusivity and coverage (recurrence) principles, without relying on known network information that has the potential to improve the discovery process of new modules. 
More recently, CoMEt~\citep{Leiserson2015} was proposed to address an inherent bias in Dendrix and Multi-Dendrix that resulted in high frequency mutations being significantly more likely to be included in mutually exclusive modules. 

However, while Dendrix, Multi-Dendrix and CoMEt have the ability to identify mutually exclusive modules {\it de novo}, there still have significant limitations. 
These methods are typically very inefficient when it comes to applying them on large-scale datasets with large parameter setting.
Also, some of these methods are randomized in nature and no guarantees exist that multiple runs of the methods will produce compatible results. Furthermore, almost all methods are able to identify only a small number of modules with limited number of genes, as cluster sizes are critical algorithmic parameters from the perspective of computational tractability.

To overcome these and other shortcomings of existing methods, we introduce a novel method called Cancer Correlation Clustering (C$^3$) to directly tackle the problems of integrating diverse sources of evidence regarding driver pattern behavior and eliminating computational bottlenecks associated with large cluster sizes or cluster numbers. C$^3$ uses an optimization framework specifically developed for the driver discovery task, where data is converted to a simple set of optimization weights that do not require the algorithm to change upon incorporation of new data sources. In addition to this flexibility, C$^3$ has low computational cost, and it allows for adding relevant problem constraints while retaining good theoretical performance guarantees. 

The paper is organized as follows. A basic introduction of the principles of correlation clustering is provided in Section~\ref{sec:approach}. Section~\ref{sec:methods} contains a description of the weight computation methods, the algorithmic clustering approach based on the computed weights, and the evaluation criteria used to compare C$^3$ and CoMEt. Section~\ref{sec:results} contains the main results of our analysis, contrasting the performance of C$^3$ and CoMEt on breast cancer and glioblastoma data. A discussion of our findings and concluding remarks are given in Section~\ref{sec:discussion}. A rigorous mathematical performance analysis of C$^3$ may be found in the Supplementary Materials.

\section{Approach} \label{sec:approach}
The basic idea behind the C$^3$ approach is \emph{correlation clustering}, an agnostic learning technique first proposed in~\citep{Bansal2004}. In the most basic form of the clustering model, one is given a set of objects and, for all or some pairs of objects, one is also given an assessment as to whether the objects are ``similar'' or ``dissimilar''. This information is described using a complete graph with labeled edges: each object is represented by a vertex of the graph, and the assessments are represented by edges labeled with either a ``+'' symbol, for similar objects, or a ``-'' symbol, for dissimilar objects. The goal is to partition the objects into clusters so that the edges within clusters are mostly positive and the edges between clusters are mostly negative. Unlike in many other clustering models, such as k-means~\citep{Hartigan1979}, the number of clusters is not fixed ahead of time and finding the optimal number of clusters is part of the problem. 

The similarity assessments need not be mutually consistent: for
example, if the graph contains a triangle with two positive edges and one negative edge, then we must either group the endpoints of the negative
edge together, erroneously putting a negative edge inside a cluster, resulting in a
``negative error'' or else we must group them separately, forcing one
of the positive edges to erroneously go between clusters, resulting in a ``positive error''. When a perfect clustering is not possible, we seek an
\emph{optimal} clustering: one that minimizes the total number of
``errors.'' This form of correlation clustering is known to be NP-hard, but depending on the graph topology, various constant or logarithmic approximation guarantees exist.

The authors of~\citep{Bansal2004} also proposed a weighted version
of the correlation-clustering problem, where the edges of the graph receive \emph{weights} between $-1$ and $1$ rather than simply
receiving $+$ or $-$ labels: an edge with weight $x$ incurs cost
$\frac{1+x}{2}$ if it is placed between clusters and cost
$\frac{1-x}{2}$ if it is placed within a cluster. A more general 
weighted formulation was introduced in~\citep{Charikar2003, Charikar2005}, and this is the formulation we subsequently consider. In this model, each edge $e$ is assigned two nonnegative weights, $\wpp_e$ and $\wmm_e$. A clustering incurs cost
$\wpp_e$ if $e$ is placed between clusters, and incurs cost $\wmm_e$
if $e$ is placed within a cluster. 

If no restrictions are placed on the weights $\wpp_e$ and $\wmm_e$,
then it is possible to have edges with $\wpp_e = \wmm_e = 0$; these
edges are effectively absent from the graph, so there is no loss of
generality in assuming that the graph is a complete graph. In order to arrive at problems that have efficient constant approximation algorithms, one needs to 
place certain restrictions on $\wpp_e$ and $\wmm_e$. The \emph{probability constraints} give a natural restriction on the edge weights $\wpp_e + \wmm_e = 1$ for every edge $e$. Another widely studied restriction is the \emph{triangle inequality} restriction, where one requires $\wmm_{uw} \leq
\wmm_{uv} + \wmm_{vw}$ for all distinct vertices $u,v$ and $w$. 

The analytic approach pursued in this work operates on the following model: genes which show sufficiently large mutation prevalence in cancer patients represent vertices of the complete graph to be clustered. 
Note that in this work we only use the top 5\% of genes ordered by mutation frequency. This equates to 170 genes in glioblastoma (GBM) and 130 genes in breast cancer (BRCA). The weights $\wpp_e$ and $\wmm_e$ of an edge $e$ connecting two genes $g_1$ and $g_2$ are weighted sums of the mutual exclusivity and coverage strength, as well as an adequately chosen measure of network distance and expression similarity. 

More precisely, the negative weights $\wmm_e$ are chosen to be relatively small if the endpoint genes describing the edge are deemed to be mutually exclusive in cancer patients. A small negative weight encourages placing mutually exclusive genes \emph{within the same cluster}. The positive weights jointly depend on the coverage, network distance and expression correlation of the endpoint genes: the larger the joint coverage, co-expression and inverse of the network distance of the endpoint genes, the more likely they will end up in the same cluster. Precise mathematical formulations of the weight functions will be provided in the next section. 

To control the size of the resulting clusters so as discourage uninformative singleton and giant clusters, we developed two new correlation clustering algorithms that use cluster sizes as problem parameters that may be chosen by the users. These cluster size bounds also allow for more accurate comparison with other methods which operate with inherent cluster size constraints. Notice that unlike in the aforementioned known methods, the cluster sizes have no bearing on the complexity of our algorithm nor on their overall approximation quality.

The driver discovery approaches closest to C$^3$ are Multi-Dendrix~\citep{Leiserson2013} and CoMEt~\citep{Leiserson2015}. Multi-Dendrix is an integer linear programming clustering algorithm that ensures that the genes within a cluster have mutation patterns that satisfy mutual exclusivity and coverage: for any two genes in a cluster, the number of patients in which these genes are mutated at the same time is relatively small; in addition, a large portion of the patients should have at least one mutation in each cluster. 
CoMEt uses a statistical score of mutation exclusive that conditions upon the frequency of each alteration, alleviating the inherent bias caused by frequently mutated genes.
Compared to Multi-Dendrix, C$^3$ uses a \emph{weighted} linear programming relaxation instead of an integer linear program which significantly improves the versatility and running time of the algorithm. Furthermore, the weights allow for straightforward incorporation of heterogeneous sources of evidence into the clustering method and the algorithm itself remains unchanged with the addition of new data. This flexibility comes at the cost of C$^3$ providing only an approximate solution. Nevertheless, given the inherently approximate nature of weight selection and parametrization of both algorithms, this does not appear to be a significant shortcoming. Also, empirical evaluations suggest that the approximation algorithms produce results very close to the optimal solution. Another advantage is that if one needs to change the combinatorial conditions that the clusters satisfy, Multi-Dendrix cannot be easily adapted, and a whole new algorithm needs to be developed. On the other hand, our algorithms are very flexible and most combinatorial patterns can be easily incorporated in the same algorithm, by simply changing the weights assigned to the edges. 

\section{Methods} \label{sec:methods}

Before rigorously describing our algorithmic methods, we introduce some relevant notation and explain how to estimate appropriate clustering weights based on available data. 

\subsection{Clustering Weights}

Let $G(V,E)$ be a complete graph, where $V(G)$ denotes the set of vertices and $E(G)$ denotes the set of edges of the graph $G$, respectively. The symbol $e\in E(G)$ or $e=uv$ with $u,v\in V(G)$ is used to denote a generic edge. Each edge is assigned a positive weight $\wpp_e$ and a negative weight $\wmm_e$. Recall the interpretation of these weights: for two distinct vertices $u,v\in V(G)$, $\wpp_{uv}$ is the cost of placing $u$ and $v$ in different clusters; consequently, by making the positive weight of an edge large, one can discourage placing the corresponding two genes into different clusters. Similarly, $\wmm_{uv}$ is the cost of placing $u$ and $v$ in the same cluster, and hence making this weight large discourages placing the corresponding two genes into the same cluster. In the rest of this section, we will explore different ways of defining the weights $\wmm_{uv}$ and $\wpp_{uv}$; in order to avoid confusion between the different definitions, each weight we define will include a parenthetical abbreviation, so that, for example, $\wpp(\text{c})_{uv}$ will refer to the positive weight of $uv$ defined according to the coverage criteria, while $\wpp(\text{c,n})_{uv}$ will refer to the positive weight of $uv$ according to the coverage and network criteria.

The weights are computed using four types of datasets: gene alteration data, gene copy number variation (CNV), network information (NI), and gene expression (GE) data. Let $n_p$ denote the number of samples (i.e., patient genomes available) and let $n_g$ denote the number of genes. Also, let $\mathbf{A}\in\{0,1\}^{n_g\times n_p}$ denote the matrix containing alteration data: if gene $i$ is altered in sample $j$, we set $\mathbf{A}(i,j)=1$; otherwise, we set $\mathbf{A}(i,j)=0$. Also, let $\mathbf{C}$ be an $n_g\times n_p$ matrix representing the CNV data: we set $\mathbf{C}(i,j)=0$ if there is no change in the copy number of gene $i$ in sample $j$; otherwise, we choose an integer value reflecting the deviation of the CNV number from its baseline value. Hence, the CNV matrix contains both positive and negative values corresponding to the copy number changes of the corresponding gene in each sample. 

To combine the CNV with alteration, we referred to the following method. Using the matrices $\mathbf{A}$ and $\mathbf{C}$, we formed a new binary matrix $\mathbf{M}\in\{0,1\}^{n_g\times n_p}$ such that
\begin{align}\label{eq:cnv1}
\mathbf{M}(i,j)=0\hspace{10pt}\textnormal{if}\hspace{10pt}\mathbf{A}(i,j)=0\  \  \textnormal{AND}\  \  l_{cnv}<\mathbf{C}(i,j)<h_{cnv},
\end{align}
and $\mathbf{M}(i,j)=1$, otherwise. In this formulation, $l_{cnv}$ and $h_{cnv}$ are lower and upper bounds that may be chosen by the user. These bounds determine what is deemed to be a significant CNV change. In our tests, we set $l_{cnv}=-1$ and $h_{cnv}=3$, although other options are clearly possible. It is worth pointing out that more conservative CNV thresholds tend to decrease coverage, while more relaxed CNV assumptions tend to decrease mutual exclusivity.
Based on the procedure above, we arrive at one ``mutation'' matrix $\mathbf{M}$ which we use instead of the matrices $\mathbf{A}$ and $\mathbf{C}$. 

Finally, let $\mathbf{Z}\in\mathbb{R}^{n_g\times n_p}$ be the matrix corresponding to $z$-scores of gene expression data: here, $\mathbf{Z}(i,j)$ denotes the $z$-score of the expression of gene $i$ in sample $j$. More precisely, if the raw expression of gene $i$ in sample $j$ equals $x_{ij}$, then $\mathbf{Z}(i,j)=\frac{x_{ij-\mu_{i}}}{\sigma_i}$; $\mu_i$ denotes the average expression of gene $i$ and $\sigma_i$ denotes its standard deviation. 

\subsubsection{Clustering Weights Determined Based on Mutual Exclusivity and Coverage (ME-CO)}\label{sec:MEC}

The idea behind our approach is to impose the mutual exclusivity constraint through the weights $\wmm_e$ and coverage constraint through the weights $\wpp_e$s. 

For each gene (i.e., vertex) $u$, let $\mathcal{S}(u)$ denote the set of patients in which $u$ is mutated. Note that we use the matrix $\mathbf{M}$ to determine if a mutation in the gene exists, either due to sequence alteration or CNV. 
Then, for any $u,v\in V(G)$, the negative weights are chosen according to
\begin{align}
\wmm(\text{e})_{u,v}=a\times\frac{|\mathcal{S}(u)\cap\mathcal{S}(v)|}{\min(|\mathcal{S}(u)|,|\mathcal{S}(v)|)},
\end{align}
where $a$ is a scaling parameter to be chosen by the user. The intuition behind the use of the factor $a$ is that if $a$ is large (e.g., empirically, a value of $a>3$ is deemed large), the mutual exclusivity is enforced strictly, while if $a$ is small, (e.g. $a<3$), the genes in each cluster may not be highly mutually exclusive. Also, note that $0\leq\frac{|\mathcal{S}(u)\cap\mathcal{S}(v)|}{\min(|\mathcal{S}(u)|,|\mathcal{S}(v)|)}\leq 1$. 

To capture the coverage property through the positive weights, observe that if two genes increase the coverage significantly, their positive weight should be large so that they are encouraged to be placed in the same cluster. To determine the positive weights, we first form the set $\mathcal{D}=\{D(u,v)\},$ for all $u,v\in V(G)$, where $D(u,v)=|\mathcal{S}(u)\  \Delta\  \mathcal{S}(v)|$ and $\Delta$ denotes the symmetric difference of two sets. A large value for the symmetric difference $D(u,v)$ suggests that the vertices $u$ and $v$ should be placed in the same cluster, since they increase the coverage of the cluster. 

Given the set $\mathcal{D}$, we define $T(J)$ to be the $J$th percentile of the values in $\mathcal{D}$. In all our runs, we used the default value of $J=95$, although this choice may be governed by the user as well. The positive weights are chosen according to:
\begin{align}
\wpp(\text{c})_{uv} &=\left\{
\begin{array}{ll} 
1,&\text{if}\  \  \  \  D(u,v)>T(J)\\
\frac{1}{T(J)} \times D(u,v)&\text{otherwise.}
\end{array}
\right.
\end{align}
Note that by this definition, $0\leq\wpp(c)_e\leq 1$ for any $e\in E(G)$. 

In order to ensure that the positive and negative weights meet the constraints imposed by our algorithm needed to ensure a constant approximation guarantee, we require that for all $u,v\in V(G)$, $\wmm(\text{e})_{uv}+\wpp(\text{c})_{uv}\geq 1$. This leads to the additional constraints:

\begin{align}
&\text{if}\  \  \  \wpp(\text{c})_{uv}+\wmm(\text{e})_{uv}<1,\\\nonumber
&\text{set} \;\,  \wmm(\text{e})_{uv}=\frac{\wmm_{uv}}{\wpp(\text{c})_{uv}+\wmm(\text{e})_{uv}},\;\,
\text{and} \;\;\; \wpp(\text{c})_{uv}=1-\wmm(\text{e})_{uv}.
\end{align}


\subsubsection{Clustering Weights Determined Based on Mutual Exclusivity, Coverage, and Network Information (NI-ME-CO)}\label{sec:MENC} The comprehensive results of pan-cancer studies reported in a number of recent papers~(\citealp{Leiserson2013},~\citealp{PortaPardo2015},~\citealp{Leiserson2015}, ~\citealp{GarciaAlonso2014}) have revealed the important connection between network topology and cancer driver distribution patterns. More precisely, the effect of deleterious mutations on the phenotype may be suppressed through a particular configuration of the corresponding protein complexes, and at the same time, the strength of the effect of a mutation may be emphasized through another configuration. As an example, most of the variants observed in healthy individuals seem to appear at the periphery of the interactome, and they do not seem to influence network connectivity. In contrast, cancer driver somatic mutations tend to occur in central, internal regions of the interactome and within highly co-integrated components. 
It appears that no previous attempts were made to more precisely quantify the network distances between driver variants, which prompted us to perform the following analysis. We first computed the pairwise (shortest) network distances between genes in a large pathway comprising $8,726$ genes from~\citet{Ciriello12} via an implementation of the standard Dijkstra algorithm~\citep{skiena1990dijkstra}. In this test, we randomly selected $1,000$ pairs in order to reduce the computational burden of running Dijkstra's algorithm $O(8726^2)$ times. By using the most complete known driver list from the Cancer Gene Census (CGC)~\citep{cgc}, we computed the same distances for driver genes, this time for all pairs of genes. The resulting distribution of shortest paths is depicted in Fig.~\ref{fig:driverdistance}. One can clearly observe that the average shortest distance between drivers is significantly smaller than the average shortest distance between two randomly selected genes. A permutation test confirms this observation, and we calculated a $p$-value of less than $0.001$.

\begin{figure}[!tpb]
\centering
\includegraphics[width=0.5\textwidth]{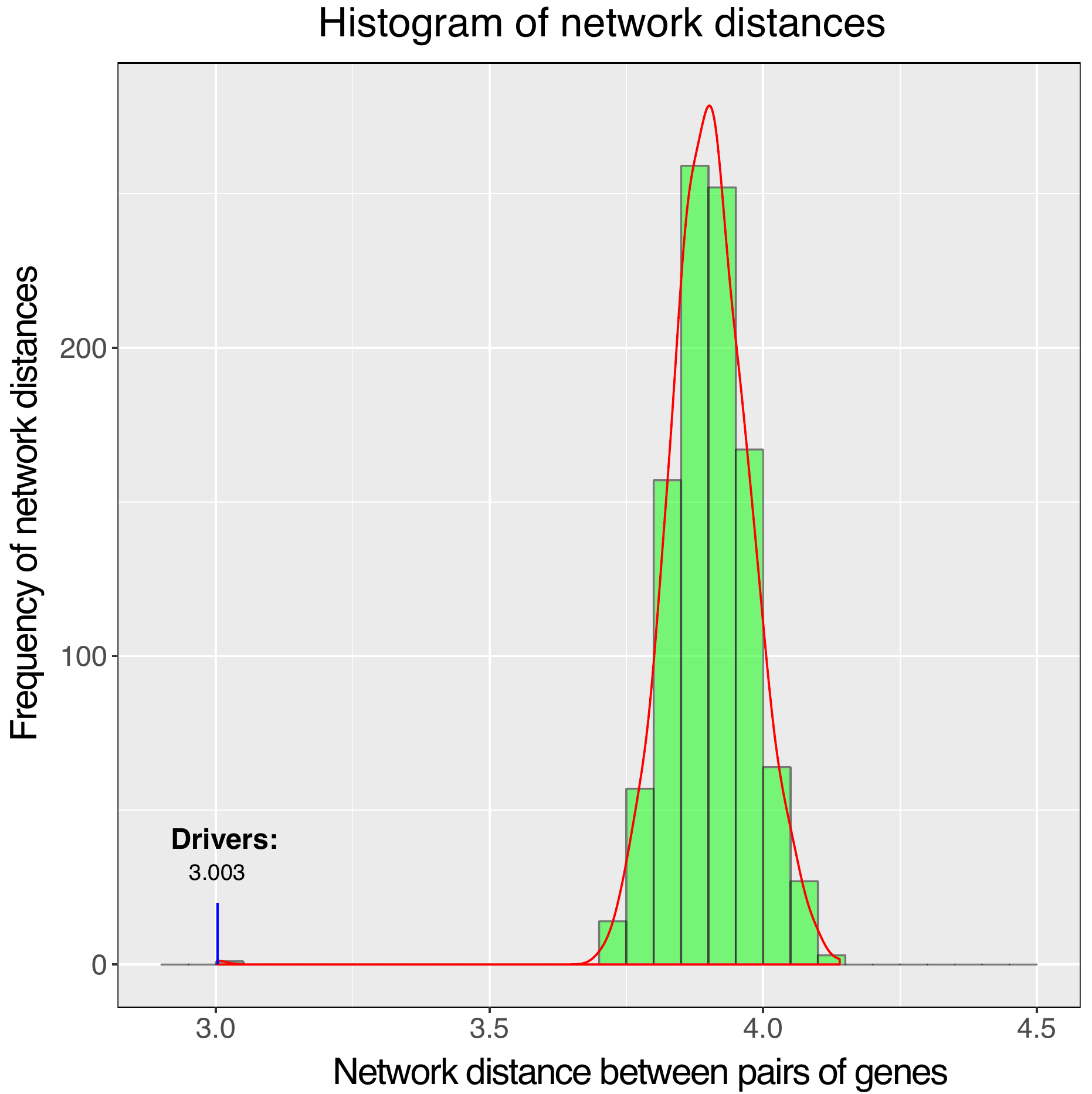} 
\caption{Histogram of shortest distances between randomly selected genes and driver genes in the network.}
\label{fig:driverdistance}
\end{figure}

These findings suggest that when determining potential driver mutations, one should make use of network distance and connectivity information. This may be accomplished in our approach by adjusting the positive weight of edges connecting two genes: if both endpoint genes were to be drivers, they should be sufficiently central to a given pathway, close to other drivers or to each other. 

For this purpose, we consider an undirected graph corresponding to the gene network, denoted by $G'$; in this graph, which is assumed to be known a priori and which in this work was retrieved from the KEGG Database, each vertex corresponds to a gene. The graph is not complete, but rather relatively sparse, and each edge represents an interaction between genes. As before, we let $n_p$ and $n_g=|V(G)|=|V(G')|$ denote the total number of patients and the total number of genes in our dataset, respectively. For each vertex $u\in V(G')$, we let $\mathcal{N}(u)$ denote the set of neighbors of $u$ and let $\mathcal{N}'(u)=\mathcal{N}(u)\cup\{u\}$. The first step in assigning the positive weights is to determine the set $\mathcal{F}=\{f(u,v)\}$, where for any pair of vertices $u,v\in V(G')$, 
\begin{align}
f(u,v)=\frac{|\mathcal{N}'(u)\cap\mathcal{N}'(v)|}{|\mathcal{N}'(u)\cup\mathcal{N}'(v)|}.
\end{align}
Note that $0\leq f(u,v)\leq 1$ for all $u,v$.
A large value of $f(u,v)$ suggests that the genes $u$ and $v$ are well connected and likely to be involved in the same pathway~\cite{Ciriello12}, and that the corresponding genes should be clustered together. 

Given the set $\mathcal{F}$, we define $T'(J')$ to be the $J'$th percentile of the values in $\mathcal{F}$. For any $u,v\in V(G)$, the positive weights are then chosen according to the following formula:
\begin{align}\label{arithmetic_mean}
\wpp(\text{c,n})_{uv} =w_1\,\wpp(\text{c})_{uv}+w_2\,\wpp(\text{n})_{uv},  
\end{align}
where $w_1,w_2 \geq 0$, $w_1+w_2=1$, and as before, the coverage weight equals
\begin{align}
\wpp(\text{c})_{uv} &=\left\{
\begin{array}{ll} 
1,&\text{if}\  \  \  \  D(u,v)>T(J)\\
\frac{1}{T(J)} \times D(u,v)&\text{otherwise,}
\end{array}
\right.
\end{align}
and the network weight equals
\begin{align}
\wpp(\text{n})_{uv} &=\left\{
\begin{array}{ll} 
1,&\text{if}\  \  \  \  f(u,v)>T'(J')\\
\frac{1}{T'(J')} \times f(u,v)&\text{otherwise.}
\end{array}
\right.
\end{align}
Again, in order to ensure that for all $u,v\in V(G)$, $\wmm(\text{e})_{uv}+\wpp(\text{c,n})_{uv}\geq 1$, we add the additional constraints
\begin{align}
&\text{if}\  \  \  \  \wpp(\text{c,n})_{uv}+\wmm(\text{e})_{uv}<1,\\\nonumber
&\text{set}\,\,  \wmm(\text{e})_{uv}=\frac{\wmm(\text{e})_{uv}}{\wpp(\text{c,n})_{uv}+\wmm(\text{e})_{uv}},\,
\text{and}\;\, \wpp(\text{c,n})_{uv}=1-\wmm(\text{e})_{uv}.
\end{align}
The weights $w_1,w_2$ may be chosen in such a way as to emphasize the importance of either coverage or network information. We suggest using $w_1=w_2=1/2$ in a coverage/network only test, although our analysis reveals that emphasizing one criterion over the other offers improved algorithm performance on some datasets.

\subsubsection{Clustering Weights Determined Based on Mutual Exclusivity, Coverage, and Gene Expression Data (EX-ME-CO)}\label{sec:EXMENC}

Similar to the case of network information, expression data may be incorporated through the positive weights, using the assumption that coexpressed genes may be involved in the same function or cancer pathway. Hence, highly coexpressed genes should be encouraged to cluster together.

To explain how to incorporate gene expression data into the clustering procedure, assume that $\mathbf{z}(u)$ and $\mathbf{z}(v)$ denote the vectors of time-evolving expression values corresponding to genes $u$ and $v$, respectively. The first step in assigning the positive weights is to determine the set $\mathcal{G}=\{g(u,v)\}$, where for every pair of genes $u,v$, 
\begin{align}
g(u,v)=\frac{|\langle\mathbf{z}(u),\mathbf{z}(v)\rangle|}{||\mathbf{z}(u)||\  ||\mathbf{z}(v)||}.
\end{align}
Here, $\langle \mathbf{a},\mathbf{b} \rangle$ denotes the classical inner product of the vectors $\mathbf{a}$ and $\mathbf{b}$, while $||\mathbf{a}||$ stands for the $\ell_2$ norm. A large value for $g(u,v)$ indicates that the expression vectors of $u$ and $v$ are highly correlated and hence should be clustered together. Also, note that $0\leq g(u,v)\leq 1$ for all $u$ and $v$. 

Given the set $\mathcal{G}$, we let $T''(J'')$ denote the $J''$th percentile of the values in $\mathcal{G}$. For any $u,v\in V(G)$, the positive weights are chosen according to the following formula:
\begin{align}\label{arithmetic_mean}
\wpp(\text{c,x})_{uv} =w_1\,\wpp(\text{c})_{uv}+w_2\,\wpp(\text{x})_{uv},  
\end{align}
where $w_1,w_2 \geq 0$, $w_1+w_2=1$, and
\begin{align}
\wpp(\text{c})_{uv} &=\left\{
\begin{array}{ll} 
1,&\text{if}\  \  \  \  D(u,v)>T(J)\\
\frac{1}{T(J)} \times D(u,v)&\text{otherwise,}
\end{array}
\right.
\end{align}
and
\begin{align}
\wpp(\text{x})_{uv} &=\left\{
\begin{array}{ll} 
1,&\text{if}\  \  \  \  g(u,v)>T''(J'')\\
\frac{1}{T''(J'')} \times g(u,v)&\text{otherwise.}
\end{array}
\right.
\end{align}
Hence, all the conditions are satisfied for the weights, except possibly the third one. In order to make sure that for all $u,v\in V(G)$, $\wmm(\text{e})_{uv}+\wpp(\text{c,x})_{uv}\geq 1$, we include an additional condition that 
\begin{align}
&\text{if}\;  \wpp(\text{c,x})_{uv}+\wmm(\text{e})_{uv}<1,\\\nonumber
&\text{set}\,\wmm(\text{e})_{uv}=\frac{\wmm(\text{e})_{uv}}{\wpp(\text{c,x})_{uv}+\wmm(\text{e})_{uv}},\,
\text{and}\, \wpp(\text{c,x})_{uv}=1-\wmm(\text{e})_{uv}.
\end{align}

Note that other combinations of datasets may be used, with appropriate changes in the weights. For example, incorporating coverage, network information as well as expression information into a positive weight may be accomplished by setting
\begin{equation}
\wpp(\text{c,n,x})_{uv}=w_1\, \wpp(\text{c})_{uv}+w_2\, \wpp(\text{n})_{uv}+w_3\, \wpp(\text{x})_{uv},
\end{equation}
where $w_1,w_2,w_3 \geq 0$, $w_1+w_2+w_3=1$.

\subsection{Clustering Algorithms}

The bounded cluster size correlation clustering problem for driver gene inference may be formulated as follows. Let $K$ be a ``hard'' bound on the size of the driver clusters, and let the positive $\wpp$ and negative weights $\wmm$ be chosen according to a desired combination of datasets, as described in the previous section. The optimum clustering may be found by solving the integer linear program (ILP) below.

\begin{align}\label{obj1}
    & \underset{x}{\text{minimize}}
    & & \sum_{e \in E(G)}(\wpp_e x_e + \wmm_e (1-x_e)) 
    \\\label{triangle1}
    & \text{subject to}
    & & x_{uv} \leq x_{uz} + x_{zv} \;\;\; {\text{(for all distinct $u,v,z \in V(G)$)}} & \\
    &&&  \sum_{v \neq u}(1 - x_{uv}) \leq K \;\;\; {\text{(for all $u \in V(G)$)}} & \\\label{int1}
    &&&  x_{e} \in \{0,1\} \;\;\; {\text{(for all $e \in E(G)$)}}. &
  \end{align}

In this formulation, and for a fixed edge $e=uv$, $x_{uv}=1$ implies that $u$ and $v$ should belong to different clusters and $x_{uv}=0$ implies that the two vertices should belong to the same cluster. Note that the triangle inequality \eqref{triangle1} ensures that if $u$ and $z$ are in the same cluster and $z$ and $v$ are in the same cluster, then $u$ and $v$ are also in the same cluster. Any clustering of the vertices can be described using the variables $x_e$. For a fixed clustering, the objective function is the cost associated with that clustering. 

Solving the ILP is NP-hard. We hence relax the problem by changing the integer constraint $x_{e} \in \{0,1\}$ to an interval constraint $x_{e} \in [0,1]$. This relaxation leads to a classical LP, the solution of which may be fractional. To obtain a valid clustering, the fractional solutions have to be subsequently \emph{rounded} to produce integer solutions. Unfortunately, known rounding algorithms~(\citealp{Puleo15}) for this problem tend to produce small clusters, often as small as single-vertex clusters. For our study, we hence slightly modify the LP by removing the cluster size constraint~(\ref{int1}), which we move directly to the rounding procedure as described in Algorithm~\ref{rounding}.   

\begin{algorithm}
\caption{}
\label{rounding}  
\begin{algorithmic}
\STATE\textbf{Input:} $\{x_e\}_{e\in E(G)}$, $\alpha$ and $K$
\STATE{Let $S = V(G)$.}
\WHILE{$S \neq \emptyset$}
\STATE{Let the ``pivot vertex'' $u$ be an arbitrary element of $S$.}
\STATE{Let $T = \{w \in S-\{u\} \st x_{uw} \leq \alpha\}$.}
\IF{$\sum_{w \in T}x_{uw} \geq \alpha\sizeof{T}/2$}
\STATE{Output the singleton cluster $\{u\}$.}
\STATE{Let $S = S-\{u\}$.}
\ELSIF{$|T|\leq K$}
\STATE{Output the cluster $\{u\} \cup T$.}
\STATE{Let $S = S - (\{u\} \cup T)$.}
\ELSE
\STATE{Partition $T$ as $T = T'_0 \cup T_1 \cup \cdots \cup T_p$, where
    $\sizeof{T'_0} = K$ and each $\sizeof{T_i} = K+1$ for $0<i< p$ and $\sizeof{T_p} \leq K+1$.}
\STATE{Let $T_0 = \{u\} \cup T'_0$.}
\STATE{Output the clusters $T_0$, $T_1,T_2,\dots,T_p$.}
\STATE{Let $S = S - (\{u\} \cup T)$.}
\ENDIF
\ENDWHILE  
\end{algorithmic}
\end{algorithm}

Algorithm~\ref{rounding} is closely based on the rounding algorithm described in ~\citet{Charikar2003,Charikar2005}.
The idea behind the rounding algorithm is to pivot on one vertex, examine its closest neighbors, where closeness is governed by the value of the output variables $x_e$ of the LP, and partition large neighborhoods if needed to get clusters of size at most $K+1$. In the Appendix of the Supplementary Materials, we prove that the LP and Rounding Algorithm~\ref{rounding} provides a $9$-approximation for the ILP problem, given that the parameter $\alpha$ is set to $2/7$ and given that the weights obey the following constraints:
\begin{itemize}
\item $\wpp_e \leq 1$ for every edge $e$, and
\item $\wpp_e + \wmm_e \geq 1$ for every edge $e$.
\end{itemize}
The above inequalities were addressed in the weight selection process through normalization, as described in the previous section.

\subsection{Evaluation methods}\label{sec:evaluation}

We evaluated the performance of both C$^3$ and CoMEt in terms of their ability to detect \emph{mutually exclusive, high-coverage, and biologically relevant gene clusters}. 
We ran both methods using mutation and CNV data collected from TCGA, pertaining to breast cancer (BRCA)~\citep{TCGA2012} and glioblastoma (GBM)~\citep{GBM}. In addition to GBM and BRCA, we also considered kidney cancer (KIRC) and ovarian cancer (OV), but the available patient data appeared limited at this stage to allow for statistically significant and comprehensive results. 
We accessed the TCGA provisional data using the cBioPortal platform~\citep{cbioportal} on August 14, 2015. 
We ran both methods using the same alteration dataset. 
We evaluated both point mutations and indels, and for CNVs, we used the GISTIC thresholds~\citep{Mermel2011} of -1 and 3 as our cut-offs. 
To focus on mutations with high frequency, we only selected genes in the top $95$ percentile of alteration frequencies, thereby obtaining $130$ genes spanning $959$ patient samples in BRCA and $170$ genes spanning $291$ patient samples in GBM.

To test the effects of cluster sizes and the quality of our results, we ran both C$^3$ and CoMEt to find clusters of sizes upper bounded by $5$, $6$, $7$, $10$, and $15$.
CoMEt is naturally designed to discover the ``most'' mutually exclusive gene sets. Due to the fact that correlation clustering will cluster all genes in a dataset, we only compared the top ten most mutually exclusive gene sets generated by C$^3$ with those of CoMEt.

We ran CoMEt with $1,000$ iterations each and $3$ initialization points to ensure both timely and consistent runs. 
For C$^3$, we ran the C$^3$ clustering method for all combinations of weights $w_1,w_2,w_3 \in \{{0,0.25,0.5,0.75,1\}}$ that satisfy $w_1+w_2+w_3=1$, but selected to report only results for the weight parameters $w_1=0.167$ (coverage), $w_2=0.333$ (network information) and $w_3=0.333$ (expression data). We observe that the choice of the weights may be completely governed by the user, and that the increase in one weight may produce better results in one performance category while reducing the performance in another category. Our choice of high weights for expression and network information was governed by the need to increase the ability of the C$^3$ algorithm to detect biologically significant clusters. Furthermore, the \emph{patient coverage} criteria appears to be less relevant than \emph{pathway coverage} and some other coverage properties that have not been explicitly investigated in the literature. 
We used four statistical methods to assess the performance of the algorithms which reflect both the statistical and biological significance of the clusters found. 

\textbf{Mutual Exclusivity.} To evaluate the degree of mutual exclusivity in a cluster, we performed a Fisher Exact Test~\citep{fisher1922mathematical} for each pair of genes in the cluster. The Fisher Exact Test uses the hypergeometric distribution to calculate the probability of observing a $2 \times 2$ contingency table of $n$ total samples, with $a$ samples that has an alteration in two genes (say, $g_i$ and $g_j$), $b$ samples with an alteration in gene $g_i$ only, and $c$ samples with an alteration in gene $g_j$ only. If $d$ is the number of samples with no alteration in either gene, then the probability of co-mutation is evaluated according to
\begin{equation}
    P(g_i,g_j) = \frac{ \binom{a + b}{a} \binom{c + d}{c}}
    {\binom{n}{a + c}}
\end{equation}
We evaluated the overall exclusivity of a cluster as the median value of each pairwise test for each pair of genes $g_i,g_j$ in the network. The pairwise Fisher's method has also been used by Mutex suite to establish mutual exclusivity~\citep{Bahur2015}. However, because in our context the Fisher Exact Test is used as an evaluation rather than discovery tool, we used the median pairwise $p$-value rather than the maximum $p$-value to get a better sense of the overall exclusivity of genes in a cluster. It is also important to note that while CoMEt has an built-in method that generalizes the exclusivity test to $2^k$ contingency table for a cluster size $k$, the exponential size of their test set makes evaluation for large cluster sizes computationally impractical. An alternative test for overall mutual exclusivity is a permutation test, as implemented by MEMo, which compares the exclusivity of a gene set by sampling random gene sets and comparing the patients with multiple alterations.

\textbf{Coverage.} To compare and evaluate the overall coverage of a cluster found by C$^3$ or CoMEt, we calculated the proportion of patients with at least one alteration in a gene belonging to the given cluster.

\textbf{Network Clustering.} We also performed a pathway analysis for the potential drivers. As pointed out in the previous section, driver genes tend to be, on average, closer to each other in a pathway compared to randomly selected genes. Using this dogma, we calculated the average pairwise distance between each each pair of genes $g_i,g_j$ within a discovered cluster using Dijkstra's Algorithm. As before, our tests were performed on 8726 genes from~\citep{Ciriello12}.

\textbf{Biological Significance.} In addition to testing the performance of the algorithm with respect to mutual exclusivity and coverage, we also investigated the biological significance of the C$^3$ and CoMEt methods from the perspective of gene discovery and pathway analysis. Although there is no overarching gold standard to determine biological significance, a commonly accepted metric employed by MEMo, Dendrix, Mutex, CoMEt and other similar tools is to count the number of known driver genes found within \emph{the best clusters} according to the mutual exclusivity principle. These clusters usually contain several known driver genes. To determine the driver gene-based biological significance, we calculated the proportion of drivers found in the ten most mutually-exclusive C$^3$ and CoMEt clusters using a comprehensive, curated list of known drivers from the Cancer Gene Census~\citep{cgc}.

It is important to point out that while the four benchmarks we define are a reliable way to test the performance of CoMEt and C$^3$, no perfect benchmark exists for detecting mutually exclusive and biologically significant genes clusters. As with many previous methods regarding this topic, the criteria are in a sense circular in that some of the same parameters that we maximize are the parameters that we use to evaluate the method. This is the reason why we use multiple benchmarks to evaluate the method.

\section{Results}\label{sec:results}
In what follows, we demonstrate that C$^3$ outperforms CoMEt in almost all of the aforementioned benchmarking criteria. As a rule of thumb, C$^3$ can be made
to outperform CoMEt \emph{in any chosen single criteria or pairs of criteria} by adjusting the weights appropriately. This observation follows from that fact that the weights trade off the strengths of different modeling assumptions. We supplement our statistical analysis with a discussion of the biological relevance of our findings, and explore the role of the new potential drivers found by C$^3$ within their driver gene communities. In particular, we discuss the significance of large mutually exclusive clusters that cannot be recovered by other methods. Recall that we restrict our attention to the ten best performing clusters according to mutual exclusivity, as this approach was used in the original evaluation process of the CoMEt algorithm.

\begin{figure*}[!tpb]
\begin{center}
\includegraphics[width=1\textwidth]{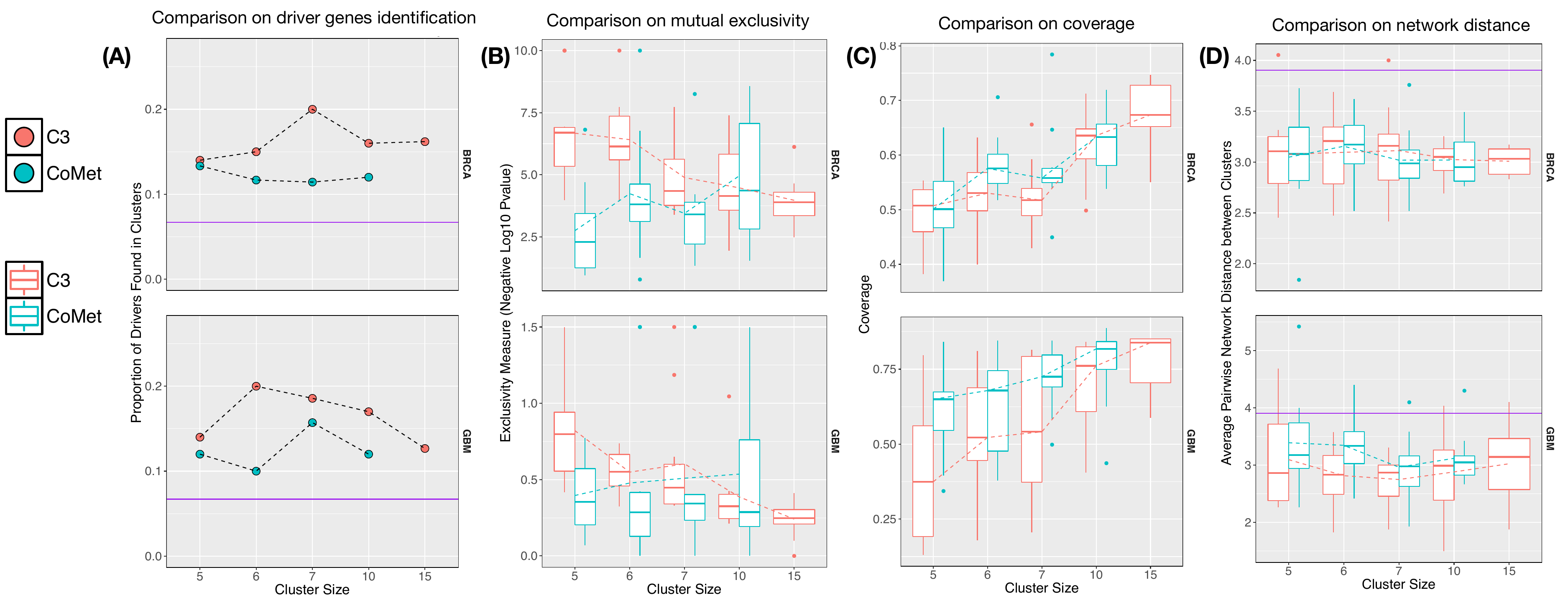} 
\end{center}
\caption{A comparative analysis of C$^3$ (Red) and CoMEt (Blue) based on four evaluation criteria. We used five cluster sizes ($5,6,7,10,$ and $15$) that index the x-axis in each benchmark test. 
\textbf{(A)} depicts the results based on the driver gene evaluation criteria. The y-axis represents the proportion of drivers found by each method, contained within the best ten  clusters found. The purple line represents the expected value of drivers detected if clusters are randomly selected. 
\textbf{(B)} shows the pairwise mutual exclusivity of each run. The y-axis represents the negative log transform of the mutual exclusive p-value such that larger values are more mutually exclusive than smaller ones. 
The boxplots illustrate the distribution of exclusivity results concerning each of the top ten individual clusters for C$^3$ and CoMEt. 
\textbf{(C)} shows the distribution of coverage, measured by proportion of samples with at least one alteration in a given cluster (the y-axis). The boxplot illustrates the distribution of coverage results for individual top ten cluster results. 
\textbf{(D)} includes the network connectivity results of C$^3$ and CoMEt. The y-axis measures the average pairwise network distance between all genes in a cluster, and the distribution of each cluster is shown in the boxplot. The purple line represents the average pairwise distance of random clusters.}
\label{fig:comet}
\end{figure*}

\subsection{Performance evaluation}\label{sec:Performance}

The results of our extensive comparison between C$^3$ and CoMEt, regarding mutual exclusivity, coverage, driver identification, and pathway-level evaluation, are shown in Fig.~\ref{fig:comet}. Both algorithms were tested on the same server with a 256GB RAM memory. Both methods ran uninterruptedly when the cluster sizes were constrained to $k=5, 6, 7,$ and $10$. CoMet reported segfault memory errors for $k=15$, and for this case, only C$^3$ was benchmarked. 

To assess the biological significance of the two methods in terms of their ability to cluster high-impact drivers from the CGC repository together, we compared the results of C$^3$ and CoMEt both to each other and to a ``baseline'' value equal to the average proportion of drivers in the ten most mutually-exclusive clusters found, in this case $0.067$, using uniform random sampling of genes (see Fig.~\ref{fig:comet}A).
In BRCA, we found that C$^3$ detected a median driver proportion of $0.160$ and CoMEt detected a median driver proportion of $0.117$ in the top ten clusters. 
C$^3$ outperforms CoMEt for each cluster size. We also used a Mann-Whitney Rank Sum test~(\citealp{rosner1999use}) to compare the overall performance of the algorithms with respect to mutual exclusivity, for all cluster sizes. 
We chose a rank-sum test because it is unclear that the drivers are following a normal distribution due to the small amount of data available. 
The results show that C$^3$ outperforms CoMEt ($p$-value of $0.0079$) in terms of amount of drivers in clusters. 
C$^3$ also outperforms CoMEt on GBM, with a median proportion of drivers per cluster equal to $0.170$, compared to a $0.12$ proportion of drivers per cluster found by CoMEt. 
This finding holds for every cluster size, with a rank-sum test $p$-value of $0.0361$. 
Both methods succeed in finding biologically significant drivers within clusters exhibiting high mutual exclusivity, and both methods significantly outperform the expected number of drivers per cluster in the random setting ($p$-value $1.594 \times 10^{-5}$ and $p$-value $1.312\times10^{-3}$ for C$^3$ and CoMEt, respectively). 

We next tested the clusters found by each method based on their mutual exclusivity (see Fig.~\ref{fig:comet}B). 
To do so, we used the previously described pairwise Fisher's exact test to obtain a p-value for each of the top ten clusters of the two methods. 
For better visualization, we performed a negative log transform on the $p$-values, and plotted the transformed $p$-value distribution. Hence, in this system, larger values indicate more mutually exclusivity. 
We again used a Mann-Whitney rank-sum test to evaluate the performance of C$^3$ and CoMEt. 
For BRCA, one can see that while both methods have significant median exclusivity values ($p = 7.541 \times 10^{-6}$ for C$^3$ and $p = 3.337 \times 10^{-4}$ for CoMEt), C$^3$ has an overall more significant p-values for each cluster size. The median $p$-value of C$^3$ for each cluster size is lower than its CoMEt counterpart except for the case $k=10$. 
However, C$^3$ does have superior performance overall with a rank-sum p-value of $p = 4.0202 \times 10^{-4}$. 
For GBM, the median exclusivity results are not as strong as for the BRCA set, for both the C$^3$ and CoMEt method. 
C$^3$ has a median $p$-value of $0.3795$ as opposed to CoMEt's $0.5022$. 
The general drop in significance may be attributed to a lower confidence of the Fisher test due to a small number of samples available; recall that the GBM set involved 291 samples, compared to 959 BRCA samples. 
This indicates that one should look at individual significant clusters to evaluate mutual exclusivity. Even for the reduced median p-value regime, C$^3$ outperforms CoMEt in significance, having lower median $p$-values for each cluster size. 
Overall, the C$^3$ $p$-values are consistently and significantly lower than those produced by CoMEt for mutual exclusivity (the rank-sum test $p$-value equals $0.04401$).

The results of the coverage tests are depicted in Fig.~\ref{fig:comet}C. 
In the coverage benchmark, CoMEt outperforms C$^3$ for GBM, but neither method outperforms the other BRCA. 
In BRCA, both methods show comparable performance, with a median result for the fraction of samples covered equal to $0.5505$ for C$^3$, and $0.5662$ for CoMEt. This rather poor performance of both methods is observed for all values of $k$, with no $p$-value based on a Student T-test~\citealp{zimmerman1987comparative} being less than $0.05$. The largest difference in coverage recorded for the two methods is present for $k=6$. 
In conclusion, there appears to be no statistical difference between C$^3$ and CoMEt in terms of BRCA coverage percentage ($p$-value of $0.5127$).
In GBM, the median $p$-value for coverage difference is more pronounced. 
The median coverage of C$^3$ is $0.632$ and the median coverage of CoMEt is $0.696$. 
CoMEt finds significantly higher-coverage clusters according to the Student T-test, with $p$-value $0.0345$, and the most pronounced coverage percentage differences exist for small values of $k$ ($0.3745$ vs. $0.6495$ for $k=5$ C$^3$ and CoMEt, respectively).
It is also important to note the wide distribution of coverage score values produced by C$^3$ for small $k$: the IQR (Interquartile range) value is roughly $0.35$ for $k=5$. 
The most likely reason behind this result is that our test weights were chosen to boost the relevance of mutual-exclusivity and biological significance rather than coverage. Mutual exclusivity accounts for 100$\%$ of the negative weights of edges, while coverage accounts for only 16.7$\%$ of the positive weights. We justify this weight choice by the fact that it leads to multiple significant cluster discovery and by our belief that coverage is a less significant driver property compared to mutual exclusivity. We also emphasize that a potentially biologically more relevant coverage constraint should pertain to important pathway, rather than patient sample coverage.

As already mentioned in the previous sections, one advantage of C$^3$ is that the user can adjust the weights according to her/his own belief about the significance of patient coverage. 
For example, by changing the averaging weights in our GBM run to $w_1=0.60$ (coverage), $w_2=0.20$ (network), and $w_3=0.20$ (expression), we obtain a coverage percentage of $0.7903$ for $k=5$. 
However, this excellent coverage comes at a cost of a less significant mutual exclusivity score (fractional value $0.4288$) and a lower proportion of detected drivers (fractional value $0.1267$). 
As may be seen from the above example, C$^3$ is highly customizable and can be adapted to the user's specification to best reflect the scope and preferences of the analysis.

The last setting in which we analyzed C$^3$ and CoMEt uses the distances of drivers in the network as performance criteria (see Fig.~\ref{fig:comet}D). 
Here, we calculated the average pairwise distance between all pairs of genes clustered together. We used the Student T-test to determine the statistical significance of this value. 
We also compared the values for both algorithms based on $1000$ randomly selected genes by using a permutation test. 
For BRCA, we found no significant performance difference between the two methods in terms of the average pairwise distance: $3.110$ for C$^3$ and $3.070$ for CoMEt, with a $p$-value of $0.9330$. 
In GBM, C$^3$ showed a smaller average pairwise distance of $2.908$ compared to CoMEt's $3.097$. 
This difference is statistically significant, with a $p$-value of $0.0379$. 
The small average network distance results of C$^3$ for GBM, coupled with the low coverage, leads to the conclusion that C$^3$ favors niche, exclusive clusters in biologically relevant cancer pathways. Hence, the method may be useful for discovering specific molecular cancer subtypes.
Both methods had an average pairwise distance well below the permutation benchmark of $3.903$: the $p$-values of both C$^3$ and CoMEt were less than $2 \times 10^{-16}$ for each cancer type.

In conclusion, from our detailed evaluation we conclude that although C$^3$ does not outperform CoMEt with respect to all four evaluation criteria, but only three of them, the C$^3$ performance indicates a strong overall propensity to select biologically more relevant and mutually exclusive clusters, and with a higher degree of flexibility compared to CoMEt. 

\subsection{Discovering potential driver pathways}\label{sec:Discovery}

We examine next the potential of the C$^3$ algorithm to detect clusters whose genes may be novel cancer driver candidates. 
We focus our search on clusters that contain biologically significant driver genes and known biological network interactions, and exhibit high mutual exclusivity and coverage. 
We also focused on the large cluster size regime, as results in this domain have not been previously reported in the literature.  
Two examples are shown in Fig.~\ref{fig:brca} and Fig.~\ref{fig:gbm}.

In BRCA, one candidate cluster with several potential novel driver genes is the cluster containing \textit{PTEN, HUWE1, CNTNAP2, GRID2, CACNA1B, CYSLTR2, MYH1} depicted in Fig.~\ref{fig:brca}. The genes in the candidate cluster are mutually exclusive ($\text{$p$-value} = 0.0084$).
The genome landscape of this cluster is dominated primarily by mutations in \textit{PTEN} and \textit{HUWE1}, and secondarily by homozygous deletions in \textit{PTEN} and \textit{CYSLTR2}. 
The most frequently altered gene in this set is a common driver gene \textit{PTEN}, a tumor suppressor gene that negatively regulates the AKT/PKB apoptosis pathway~\citep{stambolic1998negative}. 
The remaining six genes in the cluster are potential driver candidates. \textit{HUWE1} is a part of the Mule multidomain complex of the HECT domain family of E3 ubiquitin ligases responsible for apoptosis suppression, DNA damage repair, and transcriptional regulation~\citep{inoue2013mule}. 
\textit{CNTNAP2} is a neurexin protein with functions in cell-to-cell adhesion and epidermal growth factor and was found to be hypomethylated in breast cancer cell lines~\citep{shann2008genome}. 
Hypomethylation and the association with epidermal growth factors, coupled with a large number of amplifications in the alteration landscape of \textit{CNTNAP2} suggest potential oncogenic functions of the gene. 
\textit{GRID2} is an ionotropic glutamate receptor that is frequently deleted in lymphomas~\citep{roy2011tumor}.
\textit{CACNA1B} codes for a N-type calcium channel which is responsible for calcium influx. Defects in the calcium influx channel can lead to alteration in the apoptosis, proliferation, migration and invasion pathways of breast cancer~\citep{azimi2014calcium}.
\textit{CYSLTR2} is a proinflammatory cysteinyl leukotriene receptor that plays a role in cancer cell differentiation and is associated with breast cancer survival rates~\citep{magnusson2011cysteinyl}. 
MYH1 is a myosin heavy chain protein that plays a role in cell signaling and pro-apotosis pathways.

Perhaps more important than the propensity of each individual gene to be a driver is the collective interaction pattern of the seven genes in the cluster in a cancer pathway. 
From Fig.~\ref{fig:brca}, it is clear that the each gene in the cluster interacts with each other in a tightly-connected community with no gene more than three nodes away when plotted in the network, using the cBioPortal visualization tool. 
The seven genes in the cluster \textit{PTEN, HUWE1, CNTNAP2, GRID2, CACNA1B, CYSLTR2, MYH1} are strong candidates to define a novel driver pathway. 
This conclusion is reinforced by the presence of high impact common drivers (\textit{TP53, MYC, AKT, and PIK3R1}) which define several important cancer pathways such as apoptosis, DNA repair, and cell cycle arrest~\citep{vazquez2008genetics,stemke2008integrative}.



We also examined a cluster containing potential cancer drivers relevant for GBM.
In GBM, we found a cluster of size $10$ with four known drivers and many potential drivers. 
The cluster includes \textit{GLI1, WNT2, BRAF, PLCG1, FAS, CREBBP, BRCA2, GLI2, PIK3R5, VAMP3} (see Fig.~\ref{fig:gbm}). 
This large cluster has a $p$-value of $0.0901$ in terms of mutual exclusivity, which is actually low as compared to other GBM clusters. 
The cluster also contains several important driver genes such as \textit{WNT2, BRAF, BRCA2} and \textit{CREBBP} which encompass pathways such as sonic hedgehog signaling, cell fate determination, cell growth and apoptosis, checkpoint activation, and DNA repair. 
Additionally, six out of the ten members are within the same compact network community (\textit{GLI1, PLCG1, FAS, CREBBP, BRCA2, PIK3R5}). 
Of these six genes, GLI1 and GLI2 are hedgehog signaling genes that are common and first isolated in glioblastoma. These genes are responsible for cell differentiation and stem cell self-renewal~\citep{clement2007hedgehog}. 
\textit{PLCG1} is involved in intracellular transduction of receptor-mediated tyrosine kinase activators, and it has been classified as a biomarker in GBM~\citep{serao2011cell}.
\textit{FAS} is a cell surface receptor that mediates apoptosis. \textit{FAS} is known as a histological hallmark of GBM, affecting both apoptosis and necrosis factors~\citep{gratas1997fas}. 
Finally, \textit{PIK3R5} is a subunit of phosphatidylinositol 3-kinases who together have important effects on cell growth, proliferation, differentiation, motility, survival and intracellular trafficking. 

\begin{figure}[!tpb]
\centering
\includegraphics[width=0.7\textwidth]{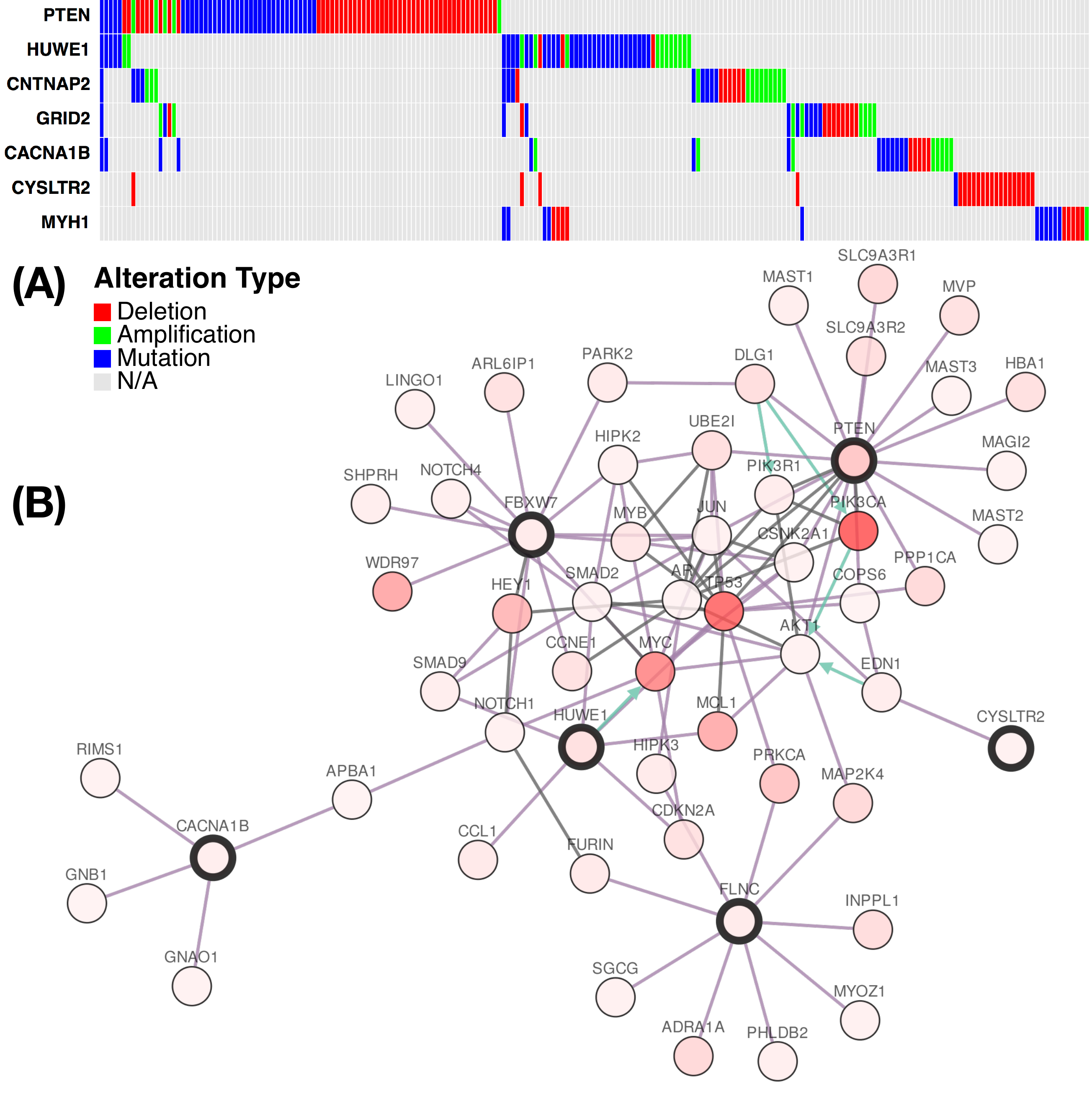} 
\caption{A cluster of potential driver genes inferred from BRCA. 
\textbf{(A)} shows the alteration landscape of the cluster, with blue representing mutation events, red representing copy number deletions, and green representing copy number amplifications.
\textbf{(B)} represents a \emph{known} subnetwork which contains 6 genes (out of 7) in \textbf{(A)}. 
The more intense the red, the higher the alteration frequency of the gene. 
Nodes highlighted in black represent driver candidates identified by C$^3$ within a small subnetwork.
Edges are depicted in black if there exists a direct interaction between two genes. 
Green edges represent an interaction that undergoes a protein state change.
Purple edges are other interactions.
}
\label{fig:brca}
\end{figure}

\begin{figure}[!tpb]
\centering
\includegraphics[width=0.6\textwidth]{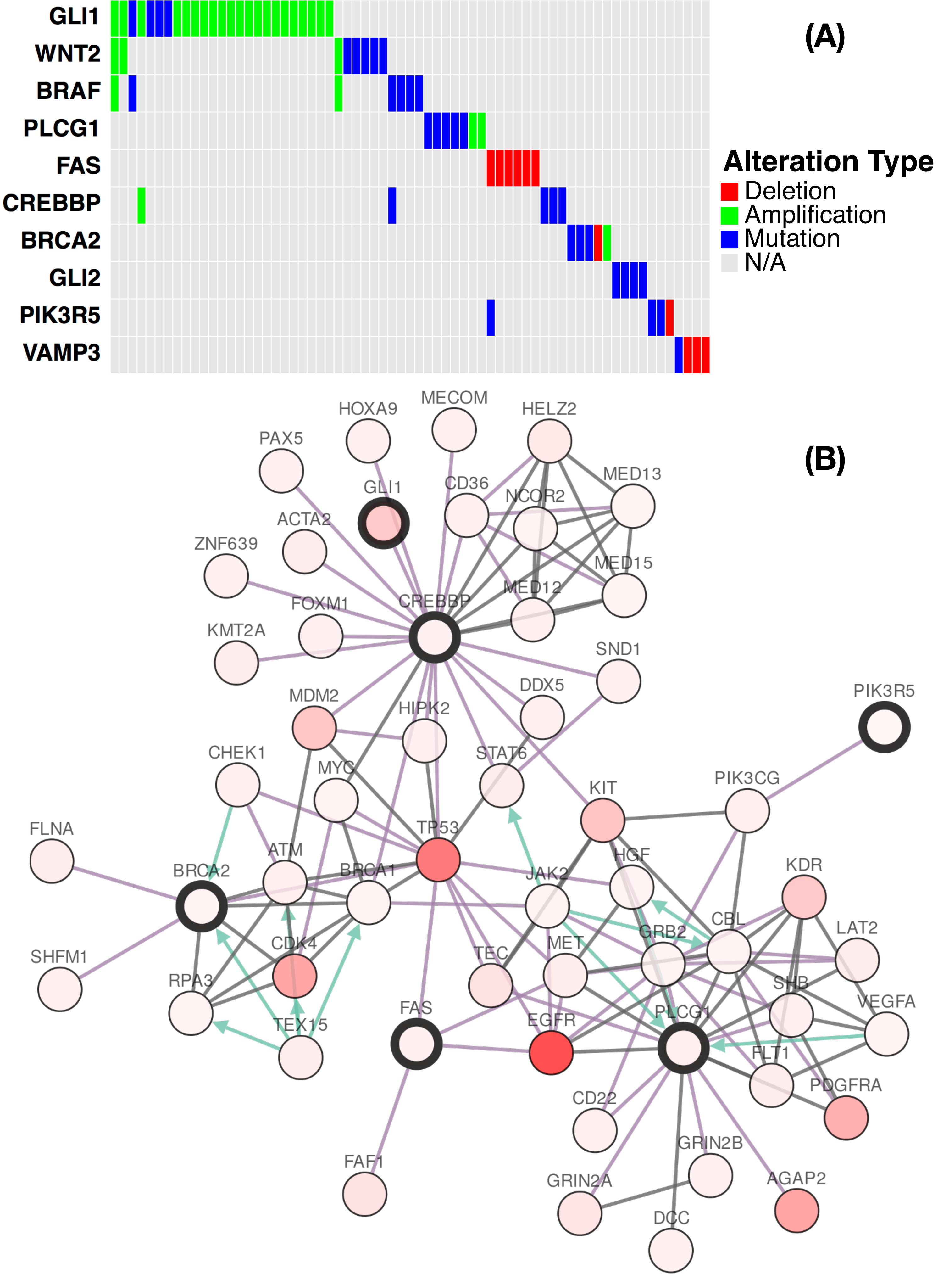}
\caption{A cluster of potential driver genes inferred from GBM. 
\textbf{(A)} shows the alteration landscape of the cluster, with blue representing mutation events, red representing copy number deletions, and green representing copy number amplifications. 
\textbf{(B)} represents a \emph{known} subnetwork which contains 6 genes (out of 10) in \textbf{(A)}. 
The more intense the red, the higher the alteration frequency of the gene. 
Nodes highlighted in black represent driver candidates identified by C$^3$ within a small subnetwork.
Edges are depicted in black if there exists a direct interaction between two genes. 
Green edges represent an interaction that undergoes a protein state change.
Purple edges are other interactions.
}
\label{fig:gbm}
\end{figure}




\section{Discussion and Conclusion} \label{sec:discussion}

We described a novel method, termed C$^3$, which has the potential to precisely and efficiently identify clusters of gene modules with mutually exclusive mutation patterns. The C$^3$ algorithm uses large-scale cancer genomics datasets which are pre-processed to yield parameters governing novel constrained correlation clustering techniques. The optimization criteria used in clustering include patterns of mutual exclusivity of mutations, patient sample coverage, and network driver concentration.

There are several major advancements of our method when compared to previously known approaches. Unlike methods that use randomized approaches without the guarantee that multiple runs of the methods on the same data will produce compatible results (such as CoMEt), C$^3$ is consistent. Also, C$^3$ has a complexity that does not depend on the chosen cluster sizes, and is hence much more appropriate for large cluster problems than other methods. Furthermore, it partitions the gene set and hence creates clusters covering all genes used in the analysis. This is to be contrasted with the results produced by other methods that tend to identify only a small number of modules with limited number of genes.

None of the previous methods were able to identify clusters utilizing different sources of information via a weighting mechanism. 
This is important because it gives us flexibility to focus more on certain aspects based on the analysis. For example, we can focus more on mutual exclusivity instead of coverage to identify clusters specific to a group of samples which may facilitate the discovery of subtype-specific modules.

By addressing the above challenges, we believe our new method C$^3$ represents a unique tool to  efficiently and reliably identify mutation patterns and driver pathways in large-scale cancer genomics studies. 

\section*{Acknowledgement}
We thank Mark Leiserson in Ben Raphael's lab at Brown University for assistance in running CoMEt.

\bibliography{document}


\newpage
\section*{Supplement -- Theoretical Performance Guarantees and Synthetic Data Evaluations}

Suppose that we have a weighted correlation-clustering instance, where 
each edge $e$ incurs a cost $\wmm_e$ if it is placed within a cluster and 
incurs a cost $\wpp_e$ if it is placed between clusters. 
We also assume that each cluster may
contain at most $K+1$ vertices, where $K$ is a fixed constant. If no
constraint on the size of the clusters is desired, one can simply set
$K = \sizeof{V}$.

We wish to find an approximation of the minimum-cost clustering. As already 
mentioned in the Approach Section, the C$^3$ method may not work for fully 
general weights, and, due to the
integrality gap in the general case, no algorithm based on LP rounding
can work in the fully general weight case. We therefore require the following 
weight constraints:
\begin{itemize}
\item $\wpp_e \leq 1$ for every edge $e$, and
\item $\wpp_e + \wmm_e \geq 1$ for every edge $e$.
\end{itemize}
Note that these constraints are satisfied if the weights obey
the probability constraints $\wpp_e + \wmm_e = 1$ for all $e$, but they
also apply to much more general choices of weights, which is relevant for 
the clustering problem at hand. 

Recall that we adopted the convention that $x_{uu}=0$ for all $u$.We also observe that one can think
of the quantity $x_{uv}$ as a ``distance'' between $u$ and $v$.
Before starting the analysis of our algorithm, we record an observation based on the triangle inequality:
\begin{observation}\label{obs}
  Let $x$ be a feasible LP solution, and let $wz$ be an edge.
  For any vertex $u$, we have $x_{wz} \geq x_{uz} - x_{uw}$ and
  $1 - x_{wz} \geq 1 - x_{uz} - x_{uw}$.
\end{observation}

Given any feasible solution $x$ to the linear program, we use the size-bounded clustering $\cee$ generated via Algorithm~\ref{rounding}, which we express in terms of a parameter $\alpha$ whose optimal value will be determined later in the proof. The idea behind the rounding is to use classical rounding to obtain a solution to the
non-size-bounded problem, and if overlarge clusters are produced, to partition them into small-sized clusters.

The standard rounding algorithm comes with a guarantee that $\cost(\cee)
\leq 6\cost(x)$.  Here, we cannot ask for such a guarantee: it is
possible, for example, that $\cost(x) = 0$ and $\cost(\cee) > 0$, if
all input edges are positive but there are too many vertices to fit
into a single cluster. Instead, we prove our approximation ratio by
bounding $\cost(\cee)$ in terms of a different lower bound on the cost
of an optimal size-bounded clustering.

Let $C_1 = \cost(x)$. Choose $Y \subset E(G)$ to minimize $\sum_{e \in
  Y}\wpp_e$ subject to the constraint that each vertex $v$ has at most
$K$ incident edges which do not lie in $Y$. Let $C_2 = \sum_{e \in
  Y}\wpp_e$. We think of $C_2$ as a lower bound on positive weight
that must be cut in order to produce a clustering with all clusters of
size at most $K+1$, since the edges contained within clusters
form a subgraph of maximum degree at most $K$. Thus, the cost of an
optimal size-bounded clustering is at least $\max\{C_1, C_2\}$. In
particular, an optimal clustering costs at least $\frac{7}{9}C_1 +
\frac{2}{9}C_2$. We will use the following lower bound on $C_2$:
\begin{definition}
  For $v \in V(G)$, the \emph{excess weight at $v$}, written $e(v)$,
  is defined by
  \[ e(v) = \min\{\sum_{z \in S}\wpp_{vz} \st \text{$S \subset N(v)$ and $\sizeof{N(v) - S} \leq K$}\}. \]
\end{definition}
\begin{lemma}\label{lem:c2bound}
  $C_2 \geq \frac{1}{2}\sum_{v \in V(G)}e(v)$.
\end{lemma}
\begin{proof}
  Choose $Y$ as described above. We have
  \[ C_2 = \sum_{e \in Y}\wpp_e = \frac{1}{2}\sum_{v \in V(G)}\sum_{vz \in Y}\wpp_{vz} \geq \frac{1}{2}\sum_{v \in V(G)}e(v), \]
  where the second equality follows from the fact that each edge of $Y$ is counted once at each of its endpoints.
\end{proof}

We use the ``charging'' idea in rounding, where one can make
``charges'' both against the individual contribution of each edge, as
well as a global ``bank'' that will be paid for using $C_2$. 

\caze{1}{A singleton cluster $\{u\}$ is output.} The total cluster
cost when outputting a singleton cluster $\{u\}$ is $\sum_{v \in S-u}
\wpp_{uv}$, while the total LP cost accrued by edges incident to $u$
is $\sum_{v \in S-u} \wpp_{uv}x_{uv}$.

If the singleton $\{u\}$ is output, then we have
\[ \sum_{v \in T}x_{uv} \geq \frac{\alpha\sizeof{T}}{2}, \] since
either the average in step $3$ was ``too high'', or else $T$ was
empty. Now for each $v \in T$ we have $x_{uv} \leq \alpha$. For such
$x_{uv}$ we have $1-x_{uv} \geq x_{uv}$, since $\alpha < 1/2$. This
yields the following lower bound on the LP cost of $uv$:
\[ \wpp_{uv} x_{uv} + \wmm_{uv} (1-x_{uv}) \geq \wpp_{uv} x_{uv} +
\wmm_{uv} x_{uv} \geq x_{uv}, \] where the last inequality uses the
bound $\wpp_{uv} + \wmm_{uv} \geq 1$. Thus, each edge $uv$ has LP
cost at least $x_{uv}$, and so the edges joining $u$ and $T$ have
total LP cost at least $\alpha\sizeof{T}/2$. Each such edge $uv$ incurs
cluster-cost $\wpp_{uv}$, which is at most $1$.  Thus, charging
$(2/\alpha)x_{uv}$ to each edge $uv$ for $v \in T$ gives enough charge to pay
for the cluster-cost of edges with $v \in T$, while each edge is
charged at most $2/\alpha$ times its LP cost.

For $v \in S-(T \cup \{u\})$, we have $x_{uv} > \alpha$, so each edge
$uv$ incurs LP cost at least $\alpha \wpp_{uv}$ and incurs cluster cost at
most $\wpp_{uv}$. Thus, charging such each edge $(1/\alpha)\wpp_{uv}x_{uv}$ pays
for the cluster cost of these edges.

\smallskip

\caze{2}{A nonsingleton cluster $\{u\} \cup T$ is output.} We first
consider edges inside $\{u\} \cup T$, then we consider edges joining
$\{u\} \cup T$ with $S - (\{u\} \cup T)$. Finally, we consider edges joining
$T_i$ with $T_j$ for $i \neq j$.

\textbf{Edges within $\{u\} \cup T$.} Suppose $vz$ is an edge contained in $\{u\}
\cup T$, so that $vz$ incurs cluster-cost $\wmm_{vz}$ and LP-cost at
least $\wmm_{vz}(1-x_{vz})$. By the definition of $T$, we have
$x_{uv}, v_{uz} \leq \alpha$.  Hence, by Observation~\ref{obs}, we have
$1-x_{vz} \geq 1 - x_{uv} - x_{vz} \geq 1-2\alpha$. Thus, charging $\frac{1}{1-2\alpha}\wmm_{vz}(1-x_{vz})$
to these edges pays for their cluster-cost.

\textbf{Edges joining $\{u\} \cup T$ with $S - (\{u\} \cup T)$.} Let
$z$ be a vertex outside $\{u\} \cup T$. A \emph{cross-edge} for $z$
is an edge from $z$ to $\{u\} \cup T$. We show that the cross-edges
for $z$ have total cluster-cost that is at most $\max\{\frac{1}{1-2\alpha}, \frac{2}{\alpha}\}$
times their total LP-cost.  Note that whenever $vz$ is a cross-edge,
we have $x_{uv} \leq \alpha$, by the definition of $T$. Each
cross-edge incurs $vz$ cluster-cost $\wpp_{vz}$ and LP-cost
$\wpp_{vz}x_{vz} + \wmm_{vz}(1-x_{vz})$.

If in fact $x_{uz} \geq 1-\alpha$, then Observation~\ref{obs} yields
$x_{vz} \geq x_{uz} - x_{uv} \geq 1-2\alpha$ for every cross-edge $v$,
yielding LP-cost at least $\frac{1}{1-2\alpha}\wpp_{vz}$; thus, such
edges have cluster-cost at most $\frac{1}{1-2\alpha}$ times their LP
cost.

It remains to handle the case $x_{uz} \in (\alpha, 1-\alpha)$. Here,
we seek a lower bound on the total LP-cost of the cross-edges for
$z$. Note that the total cluster cost of these edges is $\sum_{v \in
  \{u\} \cup T}\wpp_{vz}$, which is at most $\sizeof{T}$ since each
$\wpp_{vz} \leq 1$.

By Observation~\ref{obs}, we have $x_{vz} \geq x_{uz}-x_{uv}$ and $1 - x_{vz} \geq 1 - x_{uz} - x_{uv}$
for each edge $vz$. It follows that the total LP-weight of the cross-edges for $z$
is at least
\[ \sum_{v \in \{u\}\cup T}\left[\wpp_{vz}(x_{uz} - x_{uv}) + \wmm_{vz}(1 - x_{uz} - x_{uv})\right], \]
which rearranges to
\[ \sum_{v \in \{u\} \cup T}\left[\wpp_{vz}x_{uz} + \wmm_{vz}(1 - x_{uz})\right] - \sum_{v \in \{u\}\cup T}\left[(\wpp_{vz} + \wmm_{vz})x_{uv}\right] \]
Using the bounds $\sum_{v \in \{u\} \cup T}x_{uv} \leq \alpha(\sizeof{T} + 1)/2$ and $x_{uv} \leq \alpha$ for $v \in T$, we bound the subtracted sum as follows:
\begin{align*}
  &\sum_{v \in \{u\} \cup T} \left[(\wpp_{vz} + \wmm_{vz})x_{uv}\right] =\\
  & \sum_{v \in \{u\} \cup T}x_{uv} + \sum_{v \in \{u\} \cup T} \left[(\wpp_{vz} + \wmm_{vz} - 1)x_{uv}\right] \\
  &\leq \frac{\alpha(\sizeof{T} + 1)}{2} + \sum_{v \in \{u\} \cup T} \left[(\wpp_{vz} + \wmm_{vz} - 1)x_{uv}\right] \\
  &= \sum_{v \in \{u\} \cup T} \left[(\wpp_{vz} + \wmm_{vz} - 1)x_{uv} + \frac{\alpha}{2}\right] \\
  &\leq \sum_{v \in \{u\} \cup T} \left[\alpha(\wpp_{vz} + \wmm_{vz}) - \frac{\alpha}{2}\right].
\end{align*}
Note that in the last inequality we rely on the fact that $\wpp_{vz} + \wmm_{vz} \geq 1$.
It follows that the total LP cost is at least
\[ \sum_{v \in \{u\} \cup T}\left[\wpp_{vz}x_{uz} + \wmm_{vz}(1 - x_{uz}) - \alpha(\wpp_{vz} + \wmm_{vz}) + \frac{\alpha}{2}\right].\]
 a linear function in $x_{uz}$. We consider the behavior of this function on the interval $(\alpha, 1-\alpha)$.

When $x_{uz} = \alpha$, the lower bound simplifies to
\[ \sum_{v \in \{u\} \cup T}\left[\alpha \wpp_{vz} + (1-\alpha)\wmm_{vz} - \alpha(\wpp_{vz} + \wmm_{vz}) + \frac{\alpha}{2}\right], \]
which is at least $\alpha(\sizeof{T}+1)/2$ since $\alpha < 1/2$ implies $(1-\alpha)\wmm_{vz} \geq \alpha\wmm_{vz}$.

When $x_{uz} = 1 - \alpha$, the lower bound simplifies to
\[ \sum_{v \in \{u\} \cup T}\left[ (1-\alpha)\wpp_{vz} + \alpha\wmm_{vz} - \alpha(\wpp_{vz} + \wmm_{vz}) + \frac{\alpha}{2}\right], \]
which is again at least $\alpha(\sizeof{T}+1)/2$ since $(1-\alpha)\wpp_{vz} \geq \alpha\wpp_{vz}$.

Thus, when $x_{uz} \in (\alpha, 1-\alpha)$, we conclude that the total
cluster-cost of the cross-edges for $z$ is at most $2/\alpha$ times
the total LP-cost of those edges. In all cases, charging
$\max\{\frac{1}{1-2\alpha}, \frac{2}{\alpha}\} $ times the LP-cost of
the edges pays for their cluster-cost.

\textbf{Edges joining $T_i$ with $T_j$ for $i \neq j$.}  These are the
trickiest edges to deal with. The problem is that these edges may have
low LP-cost and high cluster-cost, so we cannot \emph{just} charge
these edges to pay for their cluster cost. Instead, we will charge
edges inside the $T_i$, and we will charge against the bank.

Let $v \in T_i$. A \emph{cross-edge} for $v$ is an edge
$vw$ with $w \in T_j$ for $j \neq i$. Note that every edge joining
$T_i, T_j$ with $i \neq j$ has at least one endpoint that does not lie
in $T_p$; as such, we will assume that $v \notin T_p$.

Define quantities $A_v$ and $B_v$ as follows:
\begin{align*}
 A_v &= \sum\{\wpp_e \st \text{$e$ is a crossing edge for $v$}\}, \\ 
 B_v &= \sum\{\wmm_{vz} \st \text{$z \in T_i-v$}\}.
\end{align*}
Since $\sizeof{T_i} = K+1$, we have
\begin{align*}
\sum_{z \in N(v)}\wpp_e \geq A_v + \sum_{z \in T_i-v}\wpp_{vz} &\geq A_v + \sum_{z \in T_i-v}(1 - \wmm_{vz})\\
&= A_v + (K - B_v),
\end{align*}
where the second inequality follows from $\wpp_{vz} + \wmm_{vz} \geq 1$ and the final equality follows from the fact that $\sizeof{T_i} = K+1$ for $i < p$. On the other hand, since each $\wpp_e \leq 1$,
we have
\[ e(v) \geq \left(\sum_{z \in N(v)}\wpp_e\right) - K \geq A_v -
B_v. \] We charge the quantity $e(v)$ to the bank, and we charge
$\wmm_{vz}$ to each edge $vz$ with $z \in T_i$. This yields total
charge at least $A_v$, which is the total cluster-cost of the
cross-edges for $v$.  Note that for each edge $vz$ charged this way,
Observation~\ref{obs} yields $1 - x_{vz} \geq 1 - x_{uv} - x_{vz} \geq
1 - 2\alpha$. Hence, $v$ charges the edge $vz$ at most
$\frac{1}{1-2\alpha}$ times its LP-cost.  Observe that each edge in
$T$ is only charged this way at its endpoints. (Furthermore, edges
whose endpoints lie in different $T_i, T_j$ with both $i,j < p$ are
actually paid for twice.)

\smallskip

In total, we have paid for all the cluster-costs by making the following charges:
\begin{itemize}
\item Edges $vz$ within each cluster $T_i$ were charged at most
  $\frac{1}{1-2\alpha}\wmm_{vz}(1-x_{vz})$ to pay for themselves plus
  at most $\frac{2}{1-2\alpha}$ to pay for edges joining $T_i$ to
  $T_j$ for $i \neq j$.  Thus, their total charge is at most
  $\frac{3}{1-2\alpha}$ times their total LP cost.
\item Edges $vz$ for which $v \in \{u\} \cup T$ and $z \in S - (\{u\} \cup T)$ were charged at most $\max\{\frac{1}{1-2\alpha}, \frac{2}{\alpha}\}$ times
  their LP cost.
\item The bank was charged $\sum_{v \in V}e(v)$. Thus, by
  Lemma~\ref{lem:c2bound}, the total charge to the bank is at
  most $2C_2$.
\end{itemize}
It follows that the total cost is minimized when
$\frac{3}{1-2\alpha} = \frac{2}{\alpha}$, which yields $\alpha = 2/7$. For this choice of $\alpha$,
the total charge is at most $7C_1 + 2C_2$, which is at most $9$ times the lower bound of $\frac{7}{9}C_1 + \frac{2}{9}C_2$.''

Given that the theoretical guarantees only establish approximation results, it is appropriate to compare the solution of the proposed relaxation with the solution of the original ILP on synthetic data. For this purpose, we created a simple graph comprising $35$ vertices grouped into five clusters of size $6$ and one cluster of size $5$: in order to ensure this cluster separation, we selected the weights as follows. For $u,v\in V(G)$, we set
\begin{align}
\wpp_{uv} &= \gamma,\  \  \  \  \text{if $u$ and $v$ are in the same cluster}\\\nonumber
\wpp_{uv} &= 1-\gamma,\  \  \  \  \text{otherwise},
\end{align} 
and $\wmm_{uv}=1-\wpp{uv}$. We ran the ILP, the C$^3$ algorithm, and the clu for $\gamma\in\{0.6,0.7,0.8,0.9,0.99\}$; all three algorithms recovered the correct clusters without making errors. 
We subsequently modified the clusters by randomly reversing the positive and negative weights of up to $20$ edges. All three algorithms recovered the correct clusters even in this scenario. 
Since large weight perturbations may change the clusters in a way that the ground truth becomes undetectable, for larger scale perturbation we decided to compare the solution of the C$^3$ algorithm to that of the ILP method. We again ran the ILP and C$^3$ algorithms on a graph with $35$ vertices, where the positive weights, $\wpp_e$, were chosen randomly according to a multinomial distribution; we also considered a number of choices for the parameters of the distributions. The negative weights were chosen according to the formula $\wmm_e=1-\wpp_e$. For each set of weights, we compared the value of the objective function for the results obtained using the two algorithms. We observed that in all cases the value of objective functions, denoted by $f$, satisfied 
\begin{align}
f(\text{ILP}) \leq f(\text{C$^3$})<2f(\text{ILP}),
\end{align} 
which shows that in practice, the C$^3$ algorithm performed much better than suggested by the theoretical analysis. Furthermore, the clusters generated by the ILP and the C$^3$ algorithm showed more than $90\%$ overlap in terms of participating nodes.

\end{document}